%% file: main.tex
\documentclass[final,leqno,onefignum,onetabnum]{siamltex1213}

\title{The Probabilistic Analysis of the Network Graph created by Dynamic Boundary Coverage} 

\author{Ganesh P. Kumar \footnotemark[1]\ \footnotemark[3]
\and Spring Berman \footnotemark[2]\ \footnotemark[3]}

\input{./Preamble}
\input{./Macros}

\begin{document}

\maketitle

\renewcommand{\thefootnote}{\fnsymbol{footnote}}

\footnotetext[2]{School of Computing, Informatics, and Decision Systems Engineering, Arizona State University, Tempe AZ 85287.}

\footnotetext[3]{School for Engineering of Matter, Transport and Energy, Arizona State University, Tempe AZ 85287.}

\footnotetext[5]{This work was supported by DARPA Young Faculty Award no. D14AP00054.}

\renewcommand{\thefootnote}{\arabic{footnote}}

\maketitle
\slugger{mms}{xxxx}{xx}{x}{x--x}

\begin{abstract}

\end{abstract}

\begin{keywords}\end{keywords}

\begin{AMS}\end{AMS}

\input{Intro}

\input{DS}

\input{IID}

\input{Renyi}

\input{Complexity}


\pagestyle{myheadings}
\thispagestyle{plain}
\markboth{TEX PRODUCTION}{USING SIAM'S \LaTeX\ MACROS}

\bibliographystyle{siam} 
\bibliography{refs,propref}

\end{document}

%% file: Preamble.tex
\let\oldtitle\title
\renewcommand\title[1]{%
    \begingroup
        \providecommand{\ttlit}{}%
        \renewcommand{\ttlit}[1]{}%
        \providecommand{\titlenote}{}%
        \renewcommand{\titlenote}[1]{}%
        \hypersetup{pdftitle={#1}}%
        \def\thetitle{#1}%
        \pdfbookmark[0]{#1}{title}
    \endgroup
    \oldtitle{#1}%
}

\usepackage{calc}

\usepackage[shortcuts]{extdash}
\usepackage{graphicx}

\usepackage{booktabs}

\usepackage{verbatim}

\usepackage{amsmath}
\usepackage{psfrag}
\usepackage{graphics} 
\usepackage{epsfig} 
\usepackage{mathptmx} 
\usepackage{times} 
\usepackage{amssymb}  
\usepackage{ifpdf}
\usepackage{cite}
\usepackage{fp}
\usepackage{algorithmicx} 
\usepackage[noend]{algpseudocode}
\usepackage{algorithm}
\usepackage{enumitem}
\usepackage{wrapfig}
\usepackage{blkarray}
\usepackage{comment}

\makeatletter
\def\citet{\@ifstar{\citetstar}{\citetnostar}}
\def\Citet{\@ifstar{\Citetstar}{\Citetnostar}}
\def\citetnostar{\@ifnextchar[{\squarecitet}{\simplecitet}}
\def\squarecitet[#1]{\@ifnextchar[{\twocitet[#1]}{\onecitet[#1]}}
\def\Citetnostar{\@ifnextchar[{\squareCitet}{\simpleCitet}}
\def\squareCitet[#1]{\@ifnextchar[{\twoCitet[#1]}{\oneCitet[#1]}}
\def\citetstar{\@ifnextchar[{\squarecitetstar}{\simplecitetstar}}
\def\squarecitetstar[#1]{\@ifnextchar[{\twocitetstar[#1]}{\onecitetstar[#1]}}
\def\Citetstar{\@ifnextchar[{\squareCitetstar}{\simpleCitetstar}}
\def\squareCitetstar[#1]{\@ifnextchar[{\twoCitetstar[#1]}{\oneCitetstar[#1]}}
\makeatother
\def\simplecitet#1{\citeauthor{#1}~\citep{#1}}
\def\onecitet[#1]#2{\citeauthor{#2}~\citep[#1]{#2}}
\def\twocitet[#1][#2]#3{\citeauthor{#3}~\citep[#1][#2]{#3}}
\def\simplecitetstar#1{\citeauthor*{#1}~\citep{#1}}
\def\onecitetstar[#1]#2{\citeauthor*{#2}~\citep[#1]{#2}}
\def\twocitetstar[#1][#2]#3{\citeauthor*{#3}~\citep[#1][#2]{#3}}
\def\simpleCitet#1{\Citeauthor{#1}~\citep{#1}}
\def\oneCitet[#1]#2{\Citeauthor{#2}~\citep[#1]{#2}}
\def\twoCitet[#1][#2]#3{\Citeauthor{#3}~\citep[#1][#2]{#3}}
\def\simpleCitetstar#1{\Citeauthor*{#1}~\citep{#1}}
\def\oneCitetstar[#1]#2{\Citeauthor*{#2}~\citep[#1]{#2}}
\def\twoCitetstar[#1][#2]#3{\Citeauthor*{#3}~\citep[#1][#2]{#3}}

\usepackage{etex} 
\usepackage[cmex10]{mathtools}             
\usepackage{amsfonts,amssymb}
\usepackage{mathrsfs}                      
\usepackage{parskip}
\usepackage{amssymb}

\usepackage{fancyref} 
\usepackage{varioref}
\usepackage[utf8]{inputenc} 
\labelformat{subfigure}{\thefigure\textup{(#1)}}
\labelformat{equation}{\textup{(#1)}}

\usepackage{booktabs}
\usepackage{pstricks}
\usepackage{float}
\usepackage{url}
\usepackage{graphicx}
\usepackage[mathscr]{eucal}
\usepackage{tikz}
\usetikzlibrary{positioning}
\usetikzlibrary{shapes.geometric}
\usetikzlibrary{shapes.misc}
\usetikzlibrary{arrows}
\usetikzlibrary{intersections}
\usepackage{amsmath}
\usepackage{amssymb}
\usepackage{xstring}
\usepackage{pgfplots} 
\makeatletter
\@ifpackageloaded{hyperref}{%
\@ifpackageloaded{amsmath}{%
\newcommand{\AMShreffix}[1]{%
        \expandafter\let\csname AMShreffix#1\expandafter\endcsname%
                \csname #1\endcsname%
        \expandafter\renewcommand\csname #1\endcsname{%
                \@hyper@itemfalse\csname AMShreffix#1\endcsname}}
\AtBeginDocument{%
\AMShreffix{equation}
\AMShreffix{align}
\AMShreffix{alignat}
\AMShreffix{flalign}
\AMShreffix{gather}
\AMShreffix{multline}
}}{}}{}
\makeatother

\def\figurename{Fig.}



%% file: Macros.tex
\newcommand{\vm}{\mathbb{V}}
\newcommand{\Vol}{\mathrm{Vol}} 
\newcommand{\favn} {\mathcal{F}}

\newcommand{\unfavn} {\mathcal{U}}
\newcommand{\favcf}{\mathcal{F}_{\mathrm{CF}}}
\newcommand{\hypn}{\mathcal{H}}

\newcommand{\upos}[1]{x_{#1}}
\newcommand{\pos}[1]{  \mathtt{x}_{#1}}
\newcommand{\uposvec}{\mathbf{x}}
\newcommand{\posvec}{\mathbf{\mathtt{x}}}

\newcommand{\psimpn} {\mathcal{P}} 
\newcommand{\psimpcf}{\mathcal{P}_{\rm{CF}}}
\newcommand{\ssimpcf}{\mathcal{S}_{\rm{CF}}}

\newcommand{\rvupos}[1] {X_{#1}}
\newcommand{\rvuposvec} {\mathbf{X}}
\newcommand{\rvpos}[1] {\mathtt{ X}_{#1}}
\newcommand{\rvposvec} { \mathbf{\mathtt{X}}}
\newcommand{\rvsvec} {\mathbf{S}}

\newcommand{\jpos} { f_{\mathtt{X}} (\posvec)}

\newcommand{\avec}{\mathbf{a}}

\newcommand{\G} {\mathcal{G}}

\newcommand{\fpar}{_{p}f}
\newcommand{\Fpar}{_{p}F} 

\newcommand{\fparvec}{_{p} \mathbf{f}}
\newcommand{\Fparvec}{_{p} \mathbf{F}}

\newcommand{\Poi}[1]{\mathrm{Poi}(#1)}

\newcommand{\num}{\mathrm{num}}
\newcommand{\den}{\mathrm{den}}

\newcommand{\PCON} {\mathsf{PCON}}

\newcommand{\PH} {\mathsf{\#PH}}
\newcommand{\PseudoP}{pseudo-$\mathsf{P}$ }

\newcommand{\EXP}{\mathsf{\#EXP}}

\newcommand{\Poly} {\mathsf{P}}

\newcommand{\VHSP}{\mathsf{VHSP}}

\newcommand{\Env}{\mathcal{E}}

\newcommand{\state}{\mathrm{state}}
 
\newcommand{\si}[1]{s_{#1}}
\newcommand{\svec} {\mathbf{s}}

\newcommand{\osi}[1] {\mathtt{s}_{#1}}
\newcommand{\orvsi}[1]{ \mathtt{S}_{#1}}
\newcommand{\osvec} {\mathbf{\mathtt{s}}}

\newcommand{\Bo}{\mathcal{B}}

\newcommand{\Inter[1]}{\mathcal{I}_{#1}}

\newcommand{\fsvec}{\tilde{\svec}}

\newcommand{\ssimpn}{\mathcal{S}}

\newcommand{\rvsla}[1] {S_{#1}}
\newcommand{\jsla} {f_{\mathbf{S}} (\mathbf{s})}

\newcommand{\rvfsla}[1] {\tilde{\rvsla{i}}}
\newcommand{\rvfpos}[1] {\tilde{\rvpos{i}}}

\newcommand{\nmin}{n_{\min}}
\newcommand{\nmax}{n_{\max}}

\newcommand{\eye}{\mathbb{I}}

\newcommand{\0}{\mathbf{0}}
\newcommand{\1}{\mathbf{1}}

\newcommand{\dia}{R}

\newcommand{\nor}{n}

\newcommand{\psat}{\pcon}
\newcommand{\sat}{\con}
\newcommand{\con}{d}
\newcommand{\coni}[1]{d_{#1}}
\newcommand{\pcon} {p_{\mathsf{con}}}

\newcommand{\rob}[1]{R_{#1}}
\newcommand{\Rob}{\mathcal{R}}

\newcommand{\Q}{\mathbb{Q}}
\newcommand{\R}{\mathbb{R}}
\newcommand{\N}{\mathbb{N}}

\newcommand{\rvn}{\mathrm{N}}
\newcommand{\parkc} {n_{pc}}

\newcommand{\vvec}{\mathbf{v}}

\newcommand{\ind} {\mathbf{I}}
\newcommand{\Uni} {\mathrm{U}}

\newcommand{\Exp} {\mathbb{E}}
 
\newcommand{\Var}[1] {\mathcal{V}(#1)}

\newcommand{\cansimp}{\Delta} 
\newcommand{\canhyp} {\mathcal{C}}

\newcommand{\meas}{\mathbf{\mathrm{Leb}}}
\newcommand{\ev}[1]{\hat{e}_{#1}}
 
\newcommand{\prob}{\mathbb{P}}
\newcommand{\rvna}{{N}_{A}}
\newcommand{\rvnd}{{N}_{D}}

\newcommand{\ddt}{\frac{\mathrm{d}}{\mathrm{d} t}}

\newcommand{\cmp}{\mathrm{cmp}}

\newcommand{\cov}{\mathrm{cov}}
\newcommand{\edg}{\mathrm{edg}}

\newtheorem{problem} {Problem}

\def\hlinewd#1{%
\noalign{\ifnum0=`}\fi\hrule \@height #1 %
\futurelet\reserved@a\@xhline}

%% file: Intro.tex
\section{Introduction}
\label{Sec:Intro} 

\subsection{Background}

We address the problem of achieving {\it boundary coverage} with a swarm of autonomous robots.  In this task, a group of robots must allocate themselves around the boundary of a region or object according to a desired configuration or density.  We specifically consider problems of {\it dynamic} boundary coverage, in which robots asynchronously  join a boundary and later leave it to recharge or perform other tasks.  Applications: mapping, exploration, environmental monitoring, surveillance, disaster response tasks such as cordoning off hazardous areas; collective payload transport, in which the group cooperatively transports a load to a destination, for automated manipulation, assembly, construction, and manufacturing.
 
We focus on \textit{stochastic coverage schemes (SCS)}, in which robots probabilistically choose positions on the boundary. Our interest in SCS, as opposed to deterministic coverage schemes, is motivated by the following reasons. First, they enable a probabilistic analysis of the graph for different classes of inputs identified by the joint pdf of robot positions.  Second, SCS allow us to model natural phenomena such as \textit{Random Sequential Adsorption} (RSA) \cite{RSACubes}, the clustering of ants around a food item \cite{KumarSHS2013}, and \textit{Renyi Parking} \cite{parking}, the process by which a fleet of cars parks without collisions on a parking lot. Lastly, results from SCS allow us to analyze the distributions of robots with noisy sensing and actuation, even though the underlying coverage scheme may be deterministic.   

\subsubsection{Assumptions about Robot Capabilities}
We assume that each robot can locally sense its environment and communicate with other robots nearby.  Disk model of sensing/communication.  Robots can distinguish between other robots and a boundary of interest.  The robots lack global localization: highly limited onboard power may preclude the use of GPS, or they may operate in GPS-denied environments.  The robots also lack prior information about their environment.  Each robots exhibits random motion that may be programmed, for instance to perform probabilistic search and tracking tasks, or that arises from inherent sensor and actuator noise.  This random motion produces uncertainty in the locations of robot encounters with a boundary.  For this reason, we refer to the task as {\it stochastic boundary coverage}.  In addition, we assume the robots have sufficient memory to store certain data structures.  



\subsection{Summary of Results}

 We devise a data structure to implement our coverage schemes, and we compute the probabilities of connectivity of various coverage schemes. 


\subsection{Models of Boundary Coverage}



We consider a team of robots, $\{\Rob_{i} \}\in [1\ldots n]$ in a bounded environment $\Env$. Robots are provided only with their (perfect) odometric readings and Wifi measurements, and a camera for detecting landmarks. Each robot is a disk of diameter $\dia$, and its Wi-fi has a coverage radius of $\sat$.    They have no knowledge of their global positions or other means to localize. In the environment is placed a load in the form of a thin line, called the \textit{Boundary} $\Bo$, which is colored black,  distinctively from the rest of the environment. One endpoint of $\Bo$ is painted white. 

Since the main thrust of this paper lies in the randomized analysis of the network created by the robots, we will make the following simplifying assumptions.  All robots are synchronized in time, with respect to a global clock. No robot fails in the course of its execution. Dealing with failures, and determining the success of boundary coverage in rugged environments where Wifi may fail are issues to be addressed in our future work. 

To begin with, we consider point robots, for which $\dia=0$ and thus the issue of inter-robot collisions does not arise. Let $\nor$ of them attached to $\Bo$ at a time instant $t \in \N$ \cite{KumarTRO2015}.   Let the position of robot $i$ be $\upos{i}$. Define the vector of \textit{unordered} positions to be $\uposvec(t):=\begin{bmatrix} \upos{1\ldots n} \end{bmatrix}\mathrm{T}$. It will be convenient to make our computations if we sort this vector in nondecreasing order to get its permutation $\posvec = \begin{bmatrix} \pos{1\ldots n} \end{bmatrix}\mathrm{T}$, whose entry $\pos{i}$ is the $i$-th robot from the left, and not necessarily the position of $\rob{i}$. Since $\pos{i}$ forms the $i$-th smallest of the $\nor$ entries of $\posvec$, it is called the $i$-th \textit{order statistic} of the positions \cite{ostat2003}. We may think of $\uposvec$ (and $\posvec$) as the realization of a PP in $\Bo^n$, so that $\uposvec$ forms a point in $\Bo^n$. Define the random variable associated with $\pos{i}$ to be $\rvpos{i}$, and place all these rv's in a vector $\rvposvec$  that defines the PP.

For convenience, we introduce two artificial robots $\pos{0} = 0$ and $\pos{\nor+1} = s$ stationed at the endpoints of $\Bo$. Since connectivity deals with inter-robot distances, it helps to think of them directly rather than in terms of $\posvec$. Define the $i$-th \textit{slack} $\si{i}$ to be the distance from $\pos{i}$ to $\pos{i+1}$, and the \textit{slack vector} $\svec$  $\svec_{1:n+1} :=  \posvec_{1:n+1} - \posvec_{0:n}$ to be the vector of all slacks.  Analogous to the rv's associated with positions, define the rv's $\rvsla{i}$ and the vector $\rvsvec$. 
We may think of $\svec$ as a point in $\Bo^{n+1}$. 

Now we introduce the notion of connectivity by defining a \textit{communication range} $d \in [0,s]$.  Two robots $\pos{i}$ and $\pos{j}$ are connected iff $|\pos{i} - \pos{j}| \leq \sat$. 
We model connectivity by a graph $\G(\posvec)$, whose nodes are $\pos{i}$ (or $\upos{i}$) , and edges are formed by pairs of connected robots. Since each node is a  geometric position, $\G(\rvposvec)$ forms a  \textit{Geometric Graph}.  We define a position vector $\posvec$ to be connected iff $\G(\posvec)$ has a path from $\pos{0}$ to $\pos{n+1}$, and disconnected otherwise.   When robot positions are chosen randomly, this graph becomes a Random Geometric Graph (RGG) \cite{Matthew2003RGG}. 

We will now make the transition from point robots to $\dia$-sized ones on the boundary. Define the \textit{position} of a robot as that of its left end, so that robot $\rob{i}$ located at $\upos{i}$ occupies the interval $\Inter{i} = [\upos{i},\upos{i}+\dia]$. The support of an attached robot position is $\Bo'=[0,s-\dia]$ (or a subset of it), so that the robot does not fall of the boundary endpoint $x=s$.  We define the position vector $\posvec$ of $n$ robots to be \textit{feasible} if there are no interrobot collisions, i.e. each slack $\si{i}$ is at least $\dia$. 

When  a robot attaches to the boundary, it selects an \textit{interval} of the boundary of length $\dia$ lying completely within the boundary. We will generally be able to abstract intervals into points, and consequently think of a SCS as the choice of multiple random \textit{points} on the boundary. Formally, a SCS is a \textit{one-dimensional Point Process (PP)} \cite{StochGeom} realized on the boundary. A special case of a PP involves robots attaching to a boundary at predefined locations.  We will be interested chiefly in the \textit{Poisson Point Process (PPP)} in which robots attach independently to the boundary, and its generalizations such as the \textit{Markov Process} \cite{StochGeom}. The independent attachments in PPPs make them easy to analyze; on the other hand, interactions between robots are harder to handle and require generalizations of PPPs.                                                                                                                                                                                                                         

To simplify our analysis, we will first work with point robots in \autoref{Sec:IID} which have $\Delta=0$, and consequently preclude inter-robot collisions. Point robots are an idealization of \textit{finite} robots which have nonzero diameter; they also provide useful approximations to the behavior of finite robots when $\Delta \ll s$. We will  compute the \textit{connectivity properties} of each SCS that we address, which include the probability of saturation, the distribution of distances to the nearest neighbor, and the joint and marginal pdfs of robot positions and inter-robot distances. 



\subsection{Problem Statement}
\label{Sec:Problem}  

We require the robots to perform the following tasks:
\begin{problem}
 \begin{enumerate}
  \item Form a connected network at the white endpoint of $\Bo$.  
  \item Attach to the boundary, forming a connected network, \textit{or} cover as much of the boundary as possible. 
  \item Efficiently the list of positions taken up by the team on $\Bo$
  \item Be able to update the map efficiently as robots join and leave the boundary
  \item Determine at any point of time the network graph, including the following properties: coverage length, number of redundant robots (i.e. those that can be removed without loss of coverage).
 \end{enumerate}
\end{problem}

\begin{problem}
 Compute the network properties for a random attachment scenario. 
\end{problem}                                                                                                                                         


\subsection{Related Work}
\label{Sec:LitRev} 

\subsubsection{Control of Multi-Robot Systems}
\label{Subsubsec:MRS}
Previous work on decentralized multi-robot boundary coverage has focused on controlling robots to converge to uniform or arbitrary formations on a circle \cite{Wang2014}.  In contrast to this work, we consider cases where there is inherent and/or programmed stochasticity in the robots' motion, and our objective is to achieve robot configurations with {\it target statistical properties}.  We assume that every robot has minimal capabilities: no global position information, and sensing or communication only within a small radius.  Task allocation strategies that are suitable for such scenarios often derive robot control policies from a continuum model of the swarm population dynamics, or {\it macroscopic model}, in order to enable the control policies to scale with the swarm size. Various stochastic approaches to robot task allocation have focused on optimizing the task-switching rates of such macroscopic models \cite{Correll08, LW10, BHHK09,OA10, MH11}. Macroscopic models have also been applied to problems of robotic assembly of products, as well as robotic self-assembly \cite{MBK09, EMM10, KBN06, NBK09}.  


\subsubsection{Wireless Networks}
\label{Subsubsec:Wifi}
Since our interest lies in getting a robot team to form a connected Wi-fi network around a boundary, we will dwell on data structures for routing in Mobile Ad Hoc Networks (\textit{MANETs}) \cite{MANET}. However, our main focus lies in analyzing the properties of the Geometric Graph (GG) formed by the multi-robot network. Our probabilistic analysis borrows heavily from the formalism of \textit{Poisson Point Processes (PPP)} \cite{RGG}, a class of spatial stochastic processes in which each robot takes positions independently of others. The network induced by a PPP is a \textit{Random Geometric Graph} (\textit{RGG}) \cite{RGG}. When attachments are required to be \textit{collison-free}), i.e. have no colliding pairs of robots, they are characterized by a Matern hard core process (\textit{HCP}). These spatial processes and their resulting RGGs have been extensively used in the wireless communication literature \cite{StochGeom}.


\subsubsection{Computational Geometry}
\label{Subsubec:CG}

All our results that compute $\pcon(G)$, the connectivity probability of $\G$, involve the order statistical properties of the pdf governing the attachment of individual robots, called the \textit{parent pdfs}, or parents for short \cite{ostat2003}. The order statistics of a collection of random variables $\mathbf{X}_{1:n}$ are generated by sorting them in nondecreasing order to get their permutation $\rvposvec$. The order statistics of uniform iid parents are the easiest to analyze; moreover,nonuniform iid parents of other forms may be readily converted to their uniform counterparts using the \textit{probability integral transform} \cite{ostat2003}. We use the computations in \cite{DyerFrieze} to derive a \PseudoP lower bound for computing $\pcon(\G)$ for uniform parents in \autoref{Subsubsec:pcon-uni}. Further, we demonstrate that determining $\pcon(\G)$ for arbitrary iid pdfs is $\PH$, by reduction from a result in \cite{DyerFrieze, dyer1991computing}. 

The analogous computation of $\pcon(\G)$ for independent, non-identical (inid) pdfs is considerably more complicated, and uses the \textit{Bapat-Beg theorem} \cite{glueck2008fast}. This computation is governed by a \cite{permanent}. We show in \autoref{Subsubsec:pcon-inid} by a reduction from the boolean permanent problem \cite{Valiant1979} that computing $\pcon$ is $\PH$.  

Our results for the Stochastic Coverage Scheme (SCS) involving Renyi parking stems from the work of Renyi \cite{Renyi58} and   Dvoretzky et al. \cite{parking}. The Renyi Parking Problem (RPP) defined in \autoref{Sec:Renyi} has been extensively studied in the physics literature under the name \textit{Random Sequential Adsorption (RSA)}, the process by which molecules get adsorbed onto a substrate surface \cite{talbot2000car,RSACubes}. The delay differential equation that governs the mean number of parked cars is extensively analyzed in \cite{RSACubes, parking}; moreover,\cite{RSACubes} computes the asymptotic properties of an interval tree that stores the occupied subintervals of the parking lot.  To our knowledge, however, there has been no analysis of the spatial probability density functions (pdfs) generated by the RPP. We derive an algorithm for computing this pdf using results from order statistics \cite{ostat2003}.

%% file: DS.tex
\section{Deterministic Coverage Strategy (DCS) for $\Bo$}
\label{Sec:tro-DS}

We will first provide a DCS for a group of finite-sized robots to form a connected network with uniform inter-robot spacing along a boundary $\Bo$.  This algorithm starts with a simple procedure, detailed in \autoref{alg:discover} below, that is guaranteed to make all robots join the same network.  Assuming there are no faults, \autoref{alg:discover} will terminate with all robots joining the network created by robot $\Rob_{1}$. 

\begin{algorithm}
    \caption{} 
    \label{alg:discover}
 \begin{algorithmic}[1]
       \Procedure{Form a connected network}{$i,n$} 
	\State $\mathtt{state} \gets \mathtt{EXPLORE}$
	\While {Black Boundary not seen}
	\State Execute Lawnmower walk
	\EndWhile 
	\While {White  endpoint not seen}
	\State Traverse $\Bo$
	\EndWhile 
	\If {$id = 1$}
	 \State Create Wifi network
	\Else
	 \While {network not created }
	   \State   Wait
	 \EndWhile
	 \State Join Wifi network
	\EndIf
	\State $\state \gets \mathtt{CONNECTED}$
	\EndProcedure
 \end{algorithmic} 
\end{algorithm}


In this MANET, \textit{every} node acts as a router. Once the robot team forms a connected network after the execution of \autoref{alg:discover}, the ID of every robot in the network is determined by flooding. This set of IDs is stored in the routing table of every robot.  Subsequently, one robot (say, $\Rob_{1}$) leaves the network to determine the length $s$ of the boundary using its odometry, and then rejoins the network by following the boundary back to its white endpoint. The maximum number of robots that can possibly attach to $\Bo$ is
\begin{align}
 \nmax = \lfloor \frac{s-R}{R} \rfloor. 
\end{align}
The minimum number of robots required to ensure connectivity is
\begin{align}
 \nmin = \lfloor \frac{s}{\con} \rfloor.
\end{align}
Based on these limits, the connectivity of the robot network falls into three categories: 

\textit{Case 1}: If $\nor < \nmin$, then at most $\nor \con$ of the boundary can be covered. 

\textit{Case 2}: If $\nor \in [\nmin,\nmax]$, then all robots can be accommodated, and can cover the boundary entirely. 

\textit{Case 3}: If $\nor > \nmax$, then $\nor - \nmax$ robots have to be dropped from coverage. In this case, the first $\nmax$ robots attach to the boundary, and the remaining are dropped.

Robots subsequently take up positions that are spaced $\con$ apart, so that $\pos{i} = (i-1)d$, using their odometry, with the white endpoint being considered $\pos{0}=0$.  Afterwards, the robots can coordinate to attach to, or detach from, $\Bo$. This DCS can be easily adapted to any SCS as follows. Instead of taking up equidistant positions, the robot team collectively samples from a joint pdf of their positions on $\Bo$, and attaches to these positions in order. The initial step of forming a connected network makes it easy to execute either coverage scheme.

\subsection{Determining properties of $\G$}
The robots can use the \textit{Optimized Link State Routing (OLSR)} protocol  \cite{OLSRv2, jacquet2001optimized} to determine the connectivity, coverage length, and number of edges of $\G$ at any instant.  OLSR is a proactive, table-driven routing protocol, each of whose nodes maintains a table of 1-hop neighbors, which are found by flooding \texttt{HELLO} messages through the network. When a new node joins or an existing one leaves the network, a set of \texttt{TC} (Topology Control) messages are initiated by the neighbors of this node, flooding the network with updated routing tables.  Robots can determine the properties of the network as follows:

\begin{enumerate}
 \item \textbf{Decide network connectivity}: Every robot floods the network with a message consisting of its id and its position. The flooding of the network is deemed to stop after a timeout $\tau$, known to all robots, at which time every robot compiles a table of robot positions and id's. From this table, the leftmost and rightmost robot IDs, $\pos{1}$ and $\pos{n}$, are identified. If $\pos{i} \leq d$ and $\pos{n} \geq s-d$, then the entire network is connected. Otherwise, each robot deems the network to be disconnected as a whole.
 \item \textbf{Number of Connected Components:} Each robot can determine its own connected component from the routing table. If the connected component of any robot covers both end-points of $\Bo$, then the network as a whole is connected. Otherwise, after a timeout period $\tau$ that is known to all robots, one robot (say $\Rob_{1}$) detaches and traverses the boundary, querying the nearest robots about their connected components. After one full traversal of $\Bo$,  $\Rob_{1}$ computes the total  number of connected components and updates. It subsequently updates other robots of this number in a second traversal of $\Bo$.
 \item \textbf{Number of edges:} This is a  variant of the approach for computing the number of connected components. Each connected component may determine the number of edges in it independently of others. As before, $\Rob_{1}$ detaches and traverses $\Bo$ to query the number of edges in each connected component, which is then broadcast to each robot. 
\end{enumerate}

\subsection{Creating and updating a list of robot positions}
Each robot $\Rob_{i}$ in the network maintains data about robot positions along the boundary in the form of an \textit{interval tree} \cite{BergCGAA}. If this position data is too large to fit into the memory of a  robot, it will keep track only of its $m$ nearest neighbors, where the size $m$ is the maximum allowable size of the tree. The interval tree handles insertions, deletions and search queries in $O(\log m)$ time.

An incoming robot that wishes to attach to $\Bo$, say $\Rob_{n+1}$,  will approach $\Bo$ and send  a broadcast query to the network to determine the locations of  slacks that are large enough for it to attach.  A subset of the attached robots will then respond to $\Rob_{n+1}$ with a list of  slacks where it may attach.  Subsequently, $\Rob_{n+1}$ attaches and broadcasts its position to its neighbors, who in turn update their position data. Likewise, an outgoing robot $\Rob_{1}$ notifies its neighbors of its impending detachment. The neighbors recompute the resulting slacks, making note of any disconnected slacks introduced by the detachment of $\Rob_{1}$. They subsequently clear $\Rob{1}$ for detachment, following which $\Rob_{1}$ detaches.

\subsection{Discussion}
This section has presented only a high level view of the BC protocol. We have deliberately processor failures, asynchrony, and anonymity for the sake of simplicity. A detailed discussion of these issues would distract from our objective of analyzing $\G$.

%% file: IID.tex
\section{IID Coverage by Homogeneous Point Robots}
\label{Sec:IID}


In this section, we consider an SCS driven by a Poisson Point Processes (PPP), in which every robot attaches independently to $\Bo$, following the same spatial parent pdf. In other words, $\rvposvec$ consists of iid random variables, and defines a PPP  on $\Bo$. Specifically, suppose that the parent pdf and cdf are $\fpar(x)$ and $\Fpar(x)$ respectively, both supported on $\Bo$. Then the number of points $\rvn$ falling on a subinterval $[a,b]$ of $\Bo$ is a Poisson random variable with underlying pdf $\fpar(x)$: 
\begin{align}
  \rvn (a,b) \sim \Poi {\lambda}~ \mathrm{where} ~ \lambda = \Fpar(b) - \Fpar(a).  
  \end{align}
  

We derive connectivity results for this SCS for a fixed team of $\nor$ robots and then generalize these results to a case of dynamic attachment and detachment. Our primary parameters of interest are the connectivity properties of $\G(\rvposvec)$, namely the probability $\pcon$ of connectivity, the expected degree of a vertex, and the number of clusters, all of which . Subsequently, we determine the \textit{spatial pdfs} of $\rvposvec$ and $\rvsvec$, both for connected and unconnected components of $\G$. 
We then apply these results to analyze the \textit{temporal properties}, such as recurrence times, of dynamic scenarios in which robots attach and detach probabilistically.   \cite{KumarICRA2014,KumarTRO2015}. 

\subsection{Geometric interpretation of connectivity}

We interpret $\posvec$ and $\svec$ as points in $\R^{n}$ and $\R^{n+1}$, respectively. The entries of $\posvec$ are nondecreasing, and thereby define the \textit{position simplex} \cite{KumarICRA2014} 
\begin{align}
 \psimpn = \{ \posvec_{1:n} : \mathrm{for ~ all} ~i:1\leq i\leq n, ~\mathrm{we~ have~ that} ~ 0\leq \pos{i} \leq \pos{i+1} \leq s \}.
\end{align}

Likewise, all valid slack vectors, i.e. those that arise from a robot configuration on $\Bo$, have entries whose sum is the boundary length $s$. Geometrically, $\svec$ defines a point on a simplex $\ssimpn$ that we call the \textit{slack simplex} \cite{KumarTRO2015}, given by 
\begin{align}
\label{eqn:ssimpn1}
 \ssimpn := \{ \svec: \1^T \svec = s , ~ \text{and}~ \0 \leq \svec \leq s\1 \} = s \cdot \cansimp_{n},
\end{align}
where
\begin{align}
 \cansimp_{n} := \{ \svec: \1^T \svec = 1 , ~ \text{and}~ \0 \leq \svec \leq \1  \}
\end{align}
is the canonical simplex in $\R^n$. The vertices of $\ssimpn$ are 
\begin{align}
 \vm(\ssimpn) = s \cdot \eye_{n+1} = s\cdot \begin{bmatrix} \ev{1} \ldots \ev{n+1} \end{bmatrix},
\end{align}

where $\ev{i}$ is the unit vector along the $i$-th axis. \autoref{eqn:ssimpn1} expresses $\ssimpn$ as a degenerate simplex, with Lebesgue measure zero in $\R^{n+1}$. For our computations, we will need to express $\ssimpn$ in full-dimensional form as
\begin{align}
\label{eqn:ssimpn2}
 \ssimpn = \{\svec_{1:n} \in \R^n :  \0^T \leq \svec , ~\text{and}~ \1^T \svec \leq s  \},
\end{align}
by dropping the last slack $\si{n+1}$, which is determined by its predecessors. Observe that all connected configurations, regardless of whether they are valid, fall within the hypercube 
\begin{align}
 \hypn := \{ 0 \leq \svec \leq d \1^T \} = d \cdot \canhyp_n,
\end{align}
where $\canhyp_n$ is the unit hypercube $[0,1]^n$.  

A valid slack vector has to lie in the intersection $\ssimpn \cap \hypn$ to represent a connected configuration. 
Define the \textit{connected region} as $\favn = \ssimpn \cap \hypn$ and the 
\textit{disconnected region} as $\unfavn := \ssimpn \setminus \favn$. We show in \cite{KumarTRO2015} that $\sat$ falls into 
three ranges, 
\begin{align}
\label{eqn:drange}
\sat \in \begin{cases}
  [0,\frac{s}{n+1}], ~\mbox{for which}~ \favn = \varnothing~ \mathrm{and}~ \psat = 0; \\
  (\frac{s}{n+1},1), ~\mbox{for which}~ \favn \subsetneq \ssimpn~ \mathrm{and}~ \psat \in 
(0,1); \\
  [1,\infty), ~ \mbox{for which}~ \favn = \ssimpn~ \mathrm{and}~ \psat = 1.
 \end{cases}
\end{align}

We may express $\favn$ as
\begin{align}
\label{eqn:fullfav}
 \favn = \{\svec_{1:n} \in \ssimpn: s-d \leq \1^T \svec_{1:n}  ~\text{and}~  
\svec \leq d \1^T    \}. 
 \end{align}
 
The parent $\fpar$ generates the joint pdf  \cite{ostat2003}
\begin{align}
\label{eqn:jpsimp}
 f_{\rvposvec} (\posvec) = n! \prod_{1\leq i\leq n} f(\pos{i}) \ind_{\psimpn}.
\end{align}
over the the position simplex, where $\ind$ denotes the \textit{indicator function} over the region in its subscript. This pdf is called the Janossy pdf \cite{StochGeom} of the PPP. Changing the argument from $\posvec$ to $\svec$ gives us
\begin{align}
 f_{\rvsvec} (\svec_{1:n}) = n! \prod_{1\leq i\leq n} f(\sum_{j=1} ^{i} s_j) \1_{\ssimpn}. 
\end{align}

\input{IID-Prop}

\input{IID-Temp}

%% file: IID-Prop.tex
We will compute the properties of $\G$ in the coming subsections, starting with $\pcon$. The formula for finding $\pcon$ provides us with a template for partitioning $\ssimpn$ into regions that are amenable to computing the following properties of $\G$: the number of connected components of $\G$,  the coverage induced by $\G$, and the edge count of $\G$. While these quantities are nontrivial to compute for  RGGs of arbitrary dimension \cite{Matthew2003RGG}, there exist straightforward, if tedious, algorithms to compute them for a single dimension. All these algorithms essentially involve computing the ratio of integrals of the joint pdf $\fsvec$ over a subset of $\ssimpn$. 

\subsection{Probability of connectivity} 
\label{Subsec:pcon}
The probability of connectivity $\psat$ is the ratio of the volume of the joint pdf lying over $\favn$ to that over $\ssimpn$:  
\begin{align}
 \psat : = \frac{ \meas(\rvsvec, \favn)} { \meas(\rvsvec,\ssimpn)} =   \frac{ \int_{\favn}  f_{\rvsvec} (\svec) d\mathbf{s} } { \int_{\ssimpn}  f_{\rvsvec} (\svec) d\mathbf{s} }.
\end{align}

where $\meas(\rvsvec,\ssimpn)$ computes the Lebesgue measure  of the joint pdf of $\rvsvec$ over $\ssimpn$. The denominator is relatively easy to evaluate analytically using barycentric coordinates \cite{gravin2012moments}, while the integral over $\favn$ is harder to compute, since there is no obvious way to decompose it into simplices.
A naive algorithm that triangulates $\favn$ into simplices will take a long time in practice when $n$ is large. Instead, we may write $\meas(\rvsvec,\favn) = \meas(\rvsvec, \ssimpn) - \meas(\rvsvec,\unfavn)$, decompose $\unfavn$ into simplices rather than $\favn$\cite{KumarICRA2014} , and finally compute $\psat = 1 -  \frac{ \meas(\rvsvec, \unfavn  )} {\meas(\rvsvec,\ssimpn )}$. 

This decomposition of $\unfavn$ will result in \textit{overlapping} simplices, whose measures we can combine using the combinatorial approach described in \cite{KumarICRA2014}.  
\begin{align}
\label{eqn:unfavsimp}
 \unfavn =   \bigcup_{\mathbf{v} \in \{0,1\}^n \setminus \{ \0 \}  } \unfavn(\vvec),   
 \end{align}

where $\unfavn(\vvec)$ forms a simplex of side $(s- \sat \1^T \mathbf{v} )$, with the vertices
\begin{align}
\label{eqn:gensimpvec}
\vm(\unfavn(\vvec)) = (s-  d \1^T \vvec) \eye_{n+1} +  \vvec.   
\end{align}

This expression is nonnegative when $s \geq d \1^T \vvec$,  so that only those vertices with at most $\nmin = \lfloor \frac{s}{d} \rfloor$  $d$'s in them need be considered. The value $\nmin$ is the minimum number of robots required for connectivity, as well as the maximum possible number of disconnected slacks. We call the simplex $\unfavn(\mathbf{v})$ the \textit{compatible simplex} of $\mathbf{v}$. Compatible simplices overlap, so the sum of their measures exceeds that of $\unfavn$. We first decompose $\unfavn$ using the inclusion-exclusion principle (IEP) as:
\begin{align}
\label{eqn:unfavmeas}
\unfavn = \bigcup_{ \mathrm{odd}~ \mathbf{v}   }  \unfavn(\mathbf{v}) \setminus \bigcup_{\mathrm{even} ~\mathbf{v} }  \unfavn(\mathbf{v}) 
\end{align}

where the (even or odd) parity of $\mathbf{v}$ is that of its number of 1-bits. We immediately have 
\begin{align}
\label{eqn:unfav2}
\meas(\rvsvec,\unfavn) = \sum _{\vvec \in \{0,1\}^{n+1} } (-1)^{\1^T \vvec}  \meas(\rvsvec, \unfavn(\mathbf{v})). 
\end{align}

Our remaining computations will rely heavily on the decomposition of $\ssimpn$ into $\unfavn(\vvec)$.

\subsection{Number of Components}
\label{Subsec:cmp}
A slack vector $\svec$ has a single connected component iff it is connected, i.e if $\svec \in \favn$.  Each unsaturated slack in $\svec$ inserts a new connected component into $\G$. Define the component counting function
\begin{align}
\label{eqn:cmp}
 \cmp : \ssimpn \mapsto \mathbb{N} ~ \text{with}~  \cmp(\svec) = \sum \limits_{i} ^{n+1} \begin{cases} 
                                                                                           1 ~ \text{if} ~ s_i \leq d \\
                                                                                           0 ~\text{otherwise}.
                                                                                          \end{cases}
\end{align}

where $\mathbb{N} = \{0,1,2,\ldots\}$. By definition, we have $\cmp = n+1-\1^T \vvec$ identically over $\unfavn(\vvec)$. We may then compute the expectation of $\cmp$ over $\ssimpn$ by 
writing $\ssimpn = \unfavn \cup \favn$, and consequently get

\begin{align}
 \Exp({\cmp}) = \meas( \cmp, \ssimpn) = \sum _{\vvec \in \{0,1\}^{n+1} : \1^T \vvec \leq \nmin } (-1)^{\1^T \vvec}  \int_{\unfavn(\vvec)} (n+1 - \1^T\vvec) \jsla d \mathbf{s}. 
\end{align}

\subsection{Coverage Length}
To determine the length of the $\Bo$ covered by $\svec$, we will introduce the \textit{coverage function} $\cov(\svec)$. If $\svec$ is connected, then its coverage length $\cov(\svec)$ is the boundary length $s$. If $\svec$ has a disconnected slack $\si{i}$, a length of $\si{i} - d$ is left without coverage. This motivates us to define $\cov$ by  
\begin{align}
\label{eqn:covlen}
 \cov : \ssimpn \mapsto \mathbb{R} ~ \text{with}~ \cov(\svec) = s - \sum \limits_{i} ^{n+1} \max(s_i-d,0).  
\end{align}

Computing $\Exp(\cov)$ over $\ssimpn$ does not get simplified by the decomposition $\ssimpn :=\unfavn \cup \favn$, for $\cov$ is non-constant over $\unfavn$. A straightforward integration gives us 
\begin{align}
 \Exp(\cov) = \meas( \cov \cdot \rvsvec, \ssimpn) = s \cdot \meas(\rvsvec,\ssimpn) - \sum_{i=1}^{n+1}  \max(0,s_i -d ) \jsla d \mathbf{s}.
\end{align}

\subsection{Number of edges of $\G$}
We will define the edge counting function $\edg(.)$ over positions rather than slacks. Given the position vector $\posvec$, there exists an edge between $\pos{i}$ and $\pos{j}$ iff $\pos{j} - \pos{i} \leq \con$. Accordingly, we have 
\begin{align}
\label{eqn:edges}
\edg : \psimpn \mapsto \mathbb{N}, ~ \text{with}~ 
 \edg(\posvec) :=  \sum_{i=1} ^{n-1} \sum_{j=i+1}^{n}   1 - \max(\pos{j} - \pos{i} -d ,0)
 \end{align}

with $\Exp(\edg)$ being the integral of \autoref{eqn:edges} over $\psimpn$:

\begin{align}
 \Exp(\edg) = \meas(\rvposvec , \psimpn) -  \int_{\psimpn} \sum_{i,j:1\leq i < j \leq n} \max(\pos{j} - \pos{i}) \jpos d\mathrm{x}.
\end{align}

%% file: IID-Temp.tex
\subsection{Dynamic coverage with iid attachment}
\label{SubSec:Conn}
Now we will examine a strategy in which the robot team dynamically attaches and detaches from the boundary, with their spatial attachment pdfs being iid on $\Bo$.

\subsubsection{Dynamic attachment}
We first consider the case in which robots attach to the boundary without detaching. One robot position is chosen at every time step using the parent pdf $\fpar$ until connectivity is achieved. We compute the expected time {until connectivity}, or the expected \textit{stopping time} of the SCS. 
To determine the stopping time, we consider the sequence $(p_i)_{i\in \N}$, where $p_i$ is the probability of connectivity with $i$ robots. Irrespective of the parent pdf $\fpar$, having more robots on $\Bo$  leads to a greater probability of connectivity. Consequently, the sequence $(p_i)$ is monotonically increasing on the support $[\nmin, \infty)$ and tends to unity as $i$ grows without bound. We also know that $p_{i:i\leq \nmin} =0$. Consequently, the attachment process will terminate (resp. fail to terminate) at $i\geq \nmin$ robots with probability $p_{i}$ (resp. $1-p_i$). The probability of connectivity being attained at $i$ robots is:
\begin{align}
\label{eqn:ptau}
 \tau_i = \begin{cases}
   0 ~~~ \mathrm{for} ~ 1\leq  i < \nmin \\
   (1-p_{i-1})  p_i ~~~  \mathrm{for} ~ i \geq  \nmin 
   \end{cases}
\end{align}

Thus, $\tau_i$ is a generalized geometric random variable whose probability of success in a trial is distinct from that in its previous one. The expected stopping time is 
\begin{align}
 \Exp(\tau) = \sum \limits_{i=\nmin} ^{\infty} i \tau_i. 
\end{align}

Since $p_{i} > p_{\nmin}$ for $i>\nmin$, we expect quicker connectivity than that of a geometric random variable whose parameter is $p_{\nmin}$ $\Exp{\tau} \leq \frac{1}{p_{\nmin}}$.

\subsubsection{Stopping time of connectivity for Uniform parent}
The uniform parent has the special property that $\rvsvec$ is jointly uniform over $\ssimpn$, with each slack being identically distributed (though not iid) as  scaled exponentials of the form $s \cdot \mathrm{Exp}(1)$. Further, the \textit{order statistics} of the slacks, represented by the vector $\osvec$, formed by sorting $\svec$ in increasing order, obey the relations \cite{ostat2003}:
\begin{align}
 \Exp({\orvsi{i}}) = \frac{s}{n+1} \sum \limits_{j=1} ^{n+1} \frac{1}{j} = \frac{s}{n+1} (H_{n+1} - H_{i}) \\
 \Var{\orvsi{i}} = \sum \limits_{j=i} ^{n+1} \frac{1}{j^2}
\end{align}
where $H_n$ denotes the harmonic numbers. The longest slack $\orvsi{n+1}$ has the expected value $\frac{s H_{n+1}}{n+1}$. To have $\osi{n+1} \leq d$, we need $\frac{H_{n+1}} {n+1} \leq \frac{d}{s}$, which may be solved numerically to get the expected hitting time of $\favn$. We may also estimate $n$ if $\nor$ is large  by approximating  $H_n$ with  $\log n$, providing
\begin{align}
\label{eqn:harmsd}
 \frac{\log(n+1)}{n+1} \leq  \frac{d}{s} \implies n = \exp(-W(\frac{d}{s})) - 1,
\end{align}
where $W$ is the \textit{Lambert W function}. 


\subsubsection{Dynamic attachment and detachment}
We now extend the results in \autoref{SubSec:Conn} to a scenario in which we require robots to strike a balance between forming a connected network on the boundary and exploring the surrounding environment of the boundary. Formally, we are given that at every time instant $t \in \R_+$, a robot may be either attached to the load or detached from it; in other words, the robot has a temporal state alphabet $\Sigma:=\{A~(\mathrm{attached}),~ D~(\mathrm{detached})\}$. 
\begin{problem}
Design the rates of switching between states, with a guarantee on the expected amount of time that the boundary will have a connected network.  
\end{problem}


To analyze the behavior of the robots, we introduce the temporal state $\rvn(t) = \begin{bmatrix} \rvna(t) & \rvnd(t) \end{bmatrix}^T$, whose entries denote the number of robots in states $A$ and $D$, respectively.   We assume that the total number of robots is conserved, which implies that 
\begin{equation}
\rvna(t) + \rvnd(t) = \rvna(0) + \rvnd(0). \label{eq:conservation} 
\end{equation}
Now we suppose that robots change state per the chemical reactions 
\begin{align}
 A  \xrightarrow{r_{AD}} D ~~\mathrm{and}~~ 
 D \xrightarrow {r_{DA}} A
\end{align}
where $r_{ij}$, the reaction rate constant, is the probability per unit time of a robot in state $i$ to switch to state $j$. The populations of robots in both states evolve over time as 
\begin{align}
\ddt \rvn(t) = \begin{bmatrix} 
                -r_{AD} & r_{DA} \\
                r_{AD} & -r_{DA}
               \end{bmatrix}  \rvn(t). 
\label{eq:ADDyn}
\end{align}
At equilibrium, $\ddt \rvn(t) = \mathbf{0}$, and so \autoref{eq:ADDyn} yields  
\begin{align}
 \frac{N_A*}{N_D*} = \frac{r_{DA}}{r_{AD}}.  \label{eq:NAND}
\end{align}
We  solve for $N_A*$ and $N_D*$ using \autoref{eq:conservation} and \autoref{eq:NAND}.

%% file: Renyi.tex
\section{Uniform Coverage by Homogeneous Finite Robots}  
\label{Sec:Renyi}   


We now consider SCS with {\it finite robots}, each of which has a nonzero diameter $\dia$.  Unlike the case of point robots, the \textit{maximum number} of attached robots is finite and given by
\begin{align}
 \nmax = \lfloor  \frac{s}{\dia} \rfloor. 
\end{align}

Collision-free positions are a realization of a \textit{Matern hard-core PP} \cite{StochGeom}, which prohibits its points from lying within a threshold distance of each other. The valid range for $\nor$ is $[\nmin,\nmax]$. When $\con \leq \dia$, every feasible configuration becomes a saturated one, causing $\nmin$ to coincide with $\nmax$. The case $\con = \dia$ is of special interest to us, since it is an instance of Renyi's Parking Problem.
\begin {problem} {Renyi's Parking Problem } \cite{RSACubes,parking}
\label{prob:renyi}
Cars of unit length park uniformly randomly on a segment of length $s$, avoiding collisions, until no parking space is available for the next car. Analyze the pmf of the {final} number of parked cars, $\rvn$. 
\end{problem}

The mean number of parked cars, $\Exp{\rvn}$, obeys a delay integral equation with the asymptotic solution
\begin{align}
 \lim _{s \rightarrow \infty} \Exp{\rvn} = \parkc \cdot  s \approx 0.748 s.
\end{align}
where $\parkc$ is Renyi's parking constant \cite{Renyi58}.  This result implies that we expect $75\%$ of the segment to be occupied by cars at the point where there is no more room to accommodate another car.  An exact solution for $\Exp{\rvn}$ leads to an intractable $\lfloor s \rfloor$-dimensional integral  \cite{effRandomPark}. Our SCS with fixed $\nor$ and uniformly random attachments is a special case of Renyi's Parking Problem in which $\rvn$ is trivial to compute.  However, to our knowledge, there has been no analysis of the spatial pdfs that are generated by the parked cars in this problem, which we provide in \autoref{Subsec:cfsimp}.

\subsection{Connectivity of Collision-Free Parking}
\label{Subsec:cfsimp}
We now formulate the CF equivalent of the point-robot attachment in \autoref{Sec:IID}. Define the \textit{position} of a robot as that of its left end, so that robot $\rob{i}$ located at $\upos{i}$ occupies the interval $[\upos{i},~\upos{i}+\dia]$. The support of all attached robot positions is $\Bo'=[0,s-\dia]$, which ensures that no robot extends beyond the boundary endpoint $x=s$. We introduce two artificial robots at $\pos{0} = -R$ and $\pos{n+1}=s$. Define $\posvec$ to be {\it collison-free (CF)} iff
\begin{equation}
\label{eqn:cfpos}
0\leq \pos{0}, ~\hspace{2mm} \pos{n}  \leq  s-R, \hspace{2mm} \text {and} ~ \pos{i} - \pos{i-1} \geq \dia,  ~\text{for} ~ i=1\ldots n,
\end{equation}

and define $\psimpcf$ to be the set of CF position vectors. Likewise define the CF subset of $\ssimpn$ and the resulting favorable region by
\begin{eqnarray}
\label{eqn:cfsvec}
\ssimpcf := \{ \svec \in \ssimpn : R\cdot \1^T \leq \svec \}  \\
\favcf : = \ssimpcf \cap \hypn =  \{ \svec \in \ssimpn : R\cdot \1^T \leq \svec \leq d \cdot \1^T \}.
\end{eqnarray}

Geometrically, $\ssimpcf$ is a simplex with the hypercuboids $\0 \leq \si{i} \leq R$ removed. Reasoning as in \autoref{Sec:IID}, we have $\pcon(\G) = \Vol(\favcf) / \Vol(\ssimpcf)$; however, we are unable to simplify this formula further as we did there.  The lack of a simplifying expression for $\pcon$ means that the computation of $\pcon(\G)$ has to involve the triangulation of $\favcf$ into simplices, a time-consuming operation that we explicitly avoided in \autoref{eqn:unfavsimp}. Likewise, expressions for the order statistics and slacks of CF positions are obtained by integrating the uniform joint pdf over $\ssimpcf$ instead of $\ssimpn$, as do the formulae for the properties of $\G$.

%% file: Complexity.tex
\section{Complexity Results of SCS}
\label{Sec:Complexity}

\subsection{Computing $\pcon$ for uniform iid parents}
\label{Subsec:pcon-uni}
We now investigate the complexity of exactly computing the integrals in \autoref{eqn:unfav2}. We begin with a \PseudoP lower bound for computing $\Vol(\favn)$, and consequently $\PCON$ for the uniform parent. We will then discuss lower bounds for non-uniform parents. Define the complexity theoretic problem $\PCON(f,s,d,n) \mapsto \pcon$, with 

\text{Input}: {Parameters of SCS }: \text{encoding of} $\fpar$, ; $s,d \in \Q_+$ ; $n \in \N$ 

\text{Output}: {Probability of connectivity}~ $\pcon \in \Q_+ $

Rational inputs and outputs are specified as exact reduced fractions; for example $s$ is input as the pair $(\num(s),\den(s))$. Let $\PCON(\Uni)$ denote that subproblem of $\PCON$ over a uniform SCS.

\begin{theorem}
$\PCON(\Uni)$ can be solved in $\Omega(n)$ and $O(n \log n)$ time.
\end{theorem}

Applying \autoref{eqn:unfav2} gives us \cite{KumarTRO2015}
\begin{align}
\label{eqn:pconuni}
\Vol(\ssimpn) = \frac{s^n \cdot \sqrt{n+1} }{ \nor!}, \\ 
\Vol(\unfavn) = \sum_{k=1}^{\nmin} (-1)^{k -1}  \binom{n+1}{k}  
\frac{(s-kd  )^{\nor} \sqrt{\nor +1} }{\nor!}, \\ \nonumber 
\PCON = \frac{\Vol(\favn)} {\Vol(\ssimpn)} = 1-\sum_{k=1}^{ \nmin 
} (-1)^{k -1}  \binom{n+1}{k} \big(1- \frac{kd}{s} \big) ^{\nor}. 
\end{align} 

\begin{proof}
\textit{Upper bound:}\label{ub}    \autoref{eqn:pconuni} is a possible 
solution for $\PCON(\Uni)$; therefore, its worst-case running time forms an upper 
bound for $\PCON(\Uni)$.  \autoref{eqn:pconuni} runs in $O(\nmin \log n)$ time;  its  
worst-case instances, which have $n > \nmin  $ take $O(n \log n)$ time, 
which forms an upper bound for any solution to $\PCON$. Note that 
\autoref{eqn:pconuni} forms a \PseudoP algorithm for $\PCON$. 

\textit{Lower bound:} Consider $\PCON(\Uni)$ instances with $s = \frac{4}{3}d$, for which 
\begin{align}
\PCON = 1 - (n+1)\cdot \frac{1}{4^n} \approx 1- \frac{n}{4^n}
\end{align}

The input size for these instances is $O(\log n)$, but the output size is 
exponential in the input size, implying an $\Omega(n)$ lower bound for any 
algorithm for $\PCON(\Uni)$. The upper bound of $O(n \log n)$ is off from the lower bound $\Omega(n)$ only by 
a polynomial in the input size of $\PCON(\Uni)$, implying that algorithms for $\PCON$ 
consume more time in outputting the solution rather than computing it. The lower bound is an exponential function of the input size,
hence $\PCON (\Uni) \in \EXP$. 
\end{proof}

\textit{Time complexity of computing $\Vol(\favn)$}: The $\sqrt{n}$ in the formula for $\Vol(\favn)$ makes it impossible to provide a bounded decimal expansion to $\Vol(\favn)$, 
hence the complexity of writing down $\Vol(\favn)$ is \textit{infinite}, except when $\nor$ is a square.  Remedying this unbounded expansion requires us to compute  $Vol(\favn) \cdot \frac{n!}{\sqrt{n+1}}$, for which the same bounds as $\PCON(\Uni)$ apply. The same bounds apply to $\Vol(\unfavn) \cdot \frac{n!} {\sqrt{n+1}}$. 

\subsection{Generalized Simplex Hypercube intersection}

In the coming sections, we will demonstrate that our problems are $\PH$ by reduction from the $\VHSP$ problem.

\begin{lemma}
\label{th:vhsp}
Define the problem $\VHSP$: 

\text{Input}: Parameters $\mathbf{a}_{1:n} , b$ of the halfspace $\mathcal{T} := \{\svec \in \R^n: \mathbf{a}^T \mathbf{s}_{1:n} \leq b \}$, with $\mathbf{a}$ and $b$ are positive rationals

\text{Output}: Volume of intersection  of $\mathcal{T}$ with the unit hypercube  $\canhyp :=[0,1]^n$ 

The solution to $\VHSP$ given by
\begin{align}
 \Vol(\mathcal{T} \cap \canhyp )= n! \prod \limits_{i=1}^{n} \sum _{\vvec \in [0,1]^n} (-1)^ {\1^T \vvec} \max( (b-\avec ^T \mathbf{s})^n , 0). 
\end{align}

is $\PH$. Equivalently, it is $\PH$ to find the probability that a random point in $\canhyp$ satisfies a single linear inequality.
\end{lemma}

It will be useful to  redefine $\VHSP$ as an intersection between a half-space with unit coefficients and a generic hypercuboid. Introduce the primed variables $s'_i := a_i s_i$, and note that $\VHSP$ asks for $\Vol(\mathcal{T}' \cup \mathcal{C}')$, where
\begin{align}
\mathcal{T}' := \{ \svec' \in \R^{n}: \1^T \svec' \leq b \} ~\text{and}~ \\ 
\mathcal{C}' := \prod [0,{a_i}]. 
\end{align}

\subsubsection{Nonuniform iid parents}
\label{Subsec:pcon-nonuni}
We will now give an example of a nonuniform parent whose $\pcon$ is $\PH$ to compute. For this purpose, define a $k$-piecewise uniform ($k$-PWU) over a finite support $[0,L]$  as follows. Partition the support into $k$ nonempty subintervals as 
\begin{align}
 [0,L]  := [0,L_1] \cup [L_1,L_2] \ldots \cup [L_{k},s]. 
\end{align}

On the $i$-th subinterval, $f(x)$ is defined to be the constant $p_i \in [0,1]$, which are chosen to satisfy $\sum p_i(L_i - L_{i-1}) = 1$, and consequently $f$ is a pdf on $\Bo$.
\begin{theorem}
\label{thm:nunihard}
 $\PCON(n\Uni)$ is $\PH$.
\end{theorem}

\begin{proof}
Given the $\VHSP$ instance with the dimension $n+1$, having the hypercuboid $\mathcal{C}' := \prod_{1\leq i\leq n+1} [0,l_i]$ and the half-space sum $b$. Let $L:=\sum l_i$. Define the equivalent instance of $\PCON$ to have the parameters 
\begin{align}
\label{eqn:npwu}
 s := \frac{b}{L} ,~ d := 1~, \text{and}~  \fpar := \frac{l_i}{L} ~\text{if} ~ x \in [i,i+1]
\end{align}

Define $Y_i \sim \Fpar(X_i)$ to be the probability integral transform of $X_i$. From the definition of $Y_i$, we have that $\prob(x_i \in [i,i+1]) = \prob ( y_i \in [0,\frac{l_i}{L})]$. It follows that if $\rvposvec$ is connected, then $\mathbf{Y}$  lies within $\mathcal{C}'$. Moreover, $\mathbf{Y}$ is jointly uniform on the half-space
\begin{align}
\mathcal{T}'=\{\mathbf{y} \in \R^n: \mathbf{y} \geq 0 ~\text{and}~ \sum y_i \leq s\}.
\end{align}

Thus, $\pcon = \Vol(\mathcal{T'} \cup \mathcal {C}') / \Vol(\mathcal{C'})$, from which the solution to $\VHSP$ can be computed in $\Poly$ time.
\end{proof}

\subsubsection{Extensions of \autoref{thm:nunihard}}
We may extend \autoref{thm:nunihard} to more general parents satisfying the constraints that:
\begin{align}
\label{eqn:intcons}
 \int \limits_{(i-1)d} ^{d}  {\fpar}(x) dx  = l_i, ~\text{for all}~ i=1,\ldots,n, ~\text{where}~ l_i \geq 0 ~ \text{and} \sum l_i = 1.
\end{align}

Since \autoref{eqn:intcons} provides us with $n$ constraint equations, $\fpar$ needs to have at least $n$ parameters to fit them, e.g. polynomials of degree $n$, with arbitrary coefficients. More generally, if $f_1,\ldots, f_n$ are arbitrary pdfs with unit supports, each having at least one parameter, their mixture
\begin{align}
 \fpar(x) = f_i(x) ~ \text{if} ~ x\in [i-1,i)  \ind_{\Bo}, ~\text{where}~ \Bo= [0,n+1]
\end{align}

may be fit to obey \autoref{eqn:intcons}. Consequently, computing $\PCON$ for this mixture is $\PH$.

On the other hand, problems with a constant number of parameters fail to admit such a reduction analogous to \autoref{thm:nunihard}, fail to be $\PH$ even though they may not exhibit an explicit formula for $\PH$. For example, we do not know a short formula for computing $\pcon$ for Renyi parking, as mentioned in \autoref{Subsec:cfsimp}. Nonetheless, the problem lacks sufficiently many free coefficients to admit a reduction, and consequently is not $\PH$. Likewise is the case of nonuniform parents such as Beta, Triangular, and clipped Gaussian pdfs on $\Bo$.

\subsubsection{Robots with heterogeneous connectivity thresholds}
\label{Subsubsec:Hetro}

We now consider a \textit{Heterogeneous} robot team whose Wifi adapters have different transmission power, so that $\Rob_{i}$ has connectivity threshold $d_i$. 
We call $\Rob_i$ \textit{weaker} (resp. \textit{stronger}) than $\Rob_j$ iff $\coni{i} < \coni{j}$ (resp. $\coni{i} > \coni{j})$. If $\coni{i} = \coni{j}$ we say that the two robots have equal power. Then the robot network is represented by the digraph $\G$, whose directed  edges are of the form $i \rightarrow j$ iff  $\Rob_{i}$ can transmit to $\Rob_{j}$. In general, edges are not bidirectional, since a weaker robot will not sense a stronger one, even though the converse holds true. Suppose without losing generality that the $d_i$'s form a strictly positive, non-decreasing sequence. 

Consider a configuration in which the robots are arranged from left to right in increasing order of their index. A connected configuration satisfies the $n+1$ constraints 
\begin{align}
 s_{2i-1},~s_{2i}   ~\leq d_i   ~ \text{for all}~ i:1,\ldots, \lfloor n/2 \rfloor , ~\text{and in addition} \\
 s_{n+1} \leq d_{n/2} ~\text{if} ~ n ~\text{is even}. \nonumber
\end{align}

In general, a connected configuration will have two distinct slacks $s_{i1}, s_{i2}$ each less than $d_i$. Further, the connected region is the intersection of $\ssimpn$ with the union of hypercuboids, each of which has two dimensions equal to $d_i$, and an extra dimension equal to $d_{n/2}$ if $n$ is even. Define $\hypn$ to be:
\begin{align}
 \hypn = \bigcup \prod [0,a_i], \text {where the }~ a_i ~\text{are a permutation of the}~ d_i.
 \end{align}

The number of hypercuboids in $\hypn$ is at most $n!$, which is the case when all $d_i$'s are distinct. We will assume that the $d_i$'s are distinct unless mentioned otherwise.  Since $\hypn$ is nonconvex in general, so is $\favn = \hypn \cap \ssimpn$. When $s$ is sufficiently large that all $\nor$ robots are required to connect it, $\favn$ becomes the disjoint union of $n!$ pieces, each of which is the intersection of $\ssimpn$ with one of the component hypercuboids of $\hypn$. Then we have that $\Vol(\favn) = n! \Vol(\ssimpn \cap \mathcal{C}')$, where $\mathcal{C}'$ is the hypercube with dimensions $d_1\times d_1\ldots \times d_n \times d_n$.

\begin{theorem}
$\PCON$ is $\PH$ for heterogeneous connectivity. 
\end{theorem}

\begin{proof}
Consider the odd-numbered instance of $\VHSP$ in dimension $2n+1$, with hypercuboid dimensions $l_1,l_1,\ldots, l_n,l_n, l_{n+1} $ and slack sum $b$. Assume that the $l_i$'s are distinct. It is clear that this instance of $\VHSP$ is at least as hard as its counterpart in $\R^n$ with dimensions $l_1,\ldots,l_n$, and thus is $\PH$. The equivalent instance of  $s=\frac{b}{L}+1$ and $(d_i)_{1\leq i\leq n} = l_i/L$, where $L=\sum l_i$ as before.  The solution of $\VHSP$ is now $n! \pcon \Vol(\mathcal{C}')$, which is computable in $\Poly$ time from the solution to $\PCON$. 

\end{proof}

It is immediately clear that finding $\Vol(\favn)$ and $\Vol(\unfavn)$ is $\PH$ for heterogeneous connectivity. 
With homogeneous robots, $\favn$ was more symmetric compared to its heterogeneous counterpart. Exploiting this symmetry led to relatively short formulae for $\PCON$ and the like. On the other hand a heterogeneous swarm is sufficiently diverse that its connected region  be an arbitrary half-space. We pay for this expressiveness by making the  connectivity problems harder. Computing $\pcon$ has a Fully Polynomial Randomized Approximation Scheme (FPRAS), which samples a uniform pdf over a subset of $\favn$ in $\Poly$,  using the Markov Chain Monte Carlo (MCMC) method \cite{JerrumCount}. Combining MCMC with Inverse CDF Sampling enables us to sample arbitrary IID pdfs over $\favn$. This approach is sufficiently general that it adapts to arbitrary joint pdfs over $\favn$, in which case it becomes the Metropolis-Hastings sampling \cite{chib1995understanding}.

\subsection{Inid parents}
\label{Subsubsec:pcon-inid}
We will finally relax the iid assumption by assuming that  position $\rvupos{i}$ has the parent pdf $\fpar_i$, and is chosen independently of others. We denote the vector of parent pdfs and cdfs by $\fparvec_{1:n}$ and $\Fparvec_{1:n}$ respectively.

\begin{theorem}
 \label{thm:inidhard}
 $\PCON$-inid, the version of $\PCON$ generated by the inid parents  $\rvuposvec \sim \fparvec_{1:n}$, where each parent is supported on $\Bo$, is $\PH$.
\end{theorem}
 
\begin{proof}
We reduce a $\PH$ subset of  $\PCON(n\Uni)$ to $\PCON$-inid. Let the given instance of $\PCON(n\Uni)$ have the boundary length $s=n+1$, and the parent defined by the $n+1$ pieces 
$\fpar_{i:1\ldots n} = p_i $ over $[i-1,i]$, with $\sum p_i = 1$. The equivalent instance of $\PCON$-inid will have $2\Uni$ parents and a boundary length $s'= n+2$. Define the $i$-th 
parent $\fpar'_{i}$ of the inid instance to be 

\begin{align}
 \fpar'_{i} = \begin{cases}
               p_i ~ \text{for} ~ x \in [i-1,i] \\
               1-p_i ~ \text{for}~ x \in [n+1,n+2] \\
               0 ~\text{elsewhere on}~ [0,n+2].
              \end{cases}
\end{align}
It is clear that $\fpar'_i$ has unit measure on its support, and is thus a pdf. Further, if the original $n\Uni$ instance is connected, then so is the inid instance \textit{on the interval } $[0,n+1]$, with connectivity on the last segment $[n+1,n+2]$ \textit{ignored}.   The unrestricted connectivity of the last segment has no effect on the complexity of the problem.
\end{proof}
